\documentclass[osajnl,twocolumn,showpacs,superscriptaddress,10pt]{revtex4-1} 
\usepackage{amsmath,amssymb,graphicx}

\usepackage{subfigure}
\usepackage{epsfig}
\usepackage{amsthm}

\newtheorem{lem}{Lemma}

\begin{document}

\title{BREN: Body Reflection Essence-Neuter Model for Separation of Reflection Components}

\author{Changsoo Je}\email{Corresponding author: vision@sogang.ac.kr}
\author{Hyung-Min Park}
\affiliation{Department of Electronic Engineering, Sogang University, 35 Baekbeom-ro, Mapo-gu, Seoul 121-742, Republic of Korea}


\begin{abstract}
We propose a novel reflection color model consisting of body essence and (mixed) neuter, and present an effective method for separating dichromatic reflection components using a single image.
Body essence is an entity invariant to interface reflection, and has two degrees of freedom unlike hue and maximum chromaticity.
As a result, the proposed method is insensitive to noise and proper for colors around CMY (cyan, magenta, and yellow) as well as RGB (red, green, and blue), contrary to the maximum chromaticity-based methods.
Interface reflection is separated by using a Gaussian function, which removes a critical thresholding problem.
Furthermore the method does not require any region segmentation.
Experimental results show the efficacy of the proposed model and method.
\end{abstract}


\maketitle 

%
%
%
%
%
%
%

\noindent
Either extremely diffuse illumination (e.g. hemispherical lighting) or Lambertian reflectance makes diffuse appearance of objects.
However in many real scenes we meet, neither illumination is sufficiently diffuse, nor object surfaces exhibit perfectly Lambertian reflection, as asserted in \cite{Gershon:1987:UCH:1625995.1626029}.
Therefore lots of scene images include specular highlights \cite{Je2013779,Je20122320}, and detecting and separating the specular interface reflection contained in images is an important task for object detection and recognition \cite{Lin_609360}, and modeling surface reflectance.
Since in general, specularity is produced by interface reflection, and body reflection is diffuse, we do not consider highly diffuse interface reflection in this Letter.

Polarization with the Brewster's angle can be used for imaging interface reflection-reduced appearance in the optical process \cite{Brewster01011815}.
Numerous techniques have been developed to detect and to separate interface reflection in the nonoptical process \cite{Artusi_CGF1971}, and many methods have employed the dichromatic reflection model \cite{Shafer_COL5080100409}, which is appropriate for the object surfaces of opaque dielectric materials.
Under the assumption of dichromatic reflection, various color-based methods have been proposed \cite{Bajcsy10.1007/BF00128233,Tan1374865}.
Baiscy et al. presented the {\bf S} space-based color reflection model and a hue-based segmentation algorithm \cite{Bajcsy10.1007/BF00128233}. Their hue-based segmentation algorithm is known erroneous in two adjacent uniform-hue regions whose saturations are different, as discussed in \cite{Tan1374865}.
Tan and Ikeuchi proposed a maximum chromaticity-based separation algorithm \cite{Tan1374865}.
Since their algorithm aligns the maximum chromaticities of two adjacent pixels in the maximum chromaticity intensity space, the remainder of chromaticity (non-maximum chromaticity) may not be well aligned. While colors around RGB (red, green, and blue) have relatively small values in channels of non-maximum chromaticity, colors around CMY (cyan, magenta, and yellow) rather have two comparable chromaticity values.
For example, while a color, $(0.9,0.03,0.07)$, which is close to R, has very small values in G and B channels, another color, $(0.47,0.05,0.48)$, which is close to M, has two comparable values in R and B channels.
Moreover the path from diffuse to specular is nonlinear in their analysis space.
Yang et al. also used the maximum chromaticity, and further introduced the coefficient of variation to make the Ch-CV space for efficient separation of reflection components \cite{Yang2013}. Their algorithm requires segmentation of specular regions, and critically depends on the segmentation result. Therefore if the segmentation is not successful, the result will be undesirable.
Shen and Cai used the modified specular-free image, in which an offset is added to their specular-free image \cite{Shen_AO09_specularity}. In their method, the diffuse component is computed by determining a single parameter that adjusts the specularity level in a whole image. Although their method is efficient, it needs region segmentation (specular/surrounding regions that consist of sufficient number of pixels), and ignores variation of diffuse component in the specular/surrounding regions. In addition, they did not rigorously justify why their simplification, only using the single adjustment parameter in a whole image, makes good results for their input images.

In this Letter, we define mixed reflectance, body neuter, body essence, and (mixed) neuter, and propose a novel reflection color model, BREN (body reflection essence-neuter) model, where mixed reflectance consists of body essence and mixed neuter.
Based on BREN model, we present an effective method for separating dichromatic reflection components using a single image.
Body essence is an entity invariant to interface reflection, and has two degrees of freedom unlike hue and maximum chromaticity.
As a result, the proposed method is insensitive to noise and proper for colors around CMY as well as RGB contrary to the maximum chromaticity-based methods.
The method computes local gradients of mixed neuter and body essence, and interface reflection is separated by using a Gaussian function, which removes a critical thresholding problem.
Moreover the method uses neither any specular-free image nor any modified one, and does not require any region segmentation.



Now we present the body reflection essence-neuter model for separating reflection components.
In the dichromatic reflection model, a reflected irradiance $L$ consists of interface ($\mathsf{f}$) and body ($\mathsf{b}$) reflection components, given as:
\begin{equation}\label{eq:dichromatic1}
L=L_\mathsf{f}+L_\mathsf{b}=m_\mathsf{f}C_\mathsf{f}+m_\mathsf{b}C_\mathsf{b},
\end{equation}
where $m_\mathsf{f}$ and $m_\mathsf{b}$ are the geometric terms, and $C_\mathsf{f}=\left(C_\mathsf{f}^R,C_\mathsf{f}^G,C_\mathsf{f}^B\right)$ and $C_\mathsf{b}=\left(C_\mathsf{b}^R,C_\mathsf{b}^G,C_\mathsf{b}^B\right)$ are the spectral terms.
From Eq.~\ref{eq:dichromatic1}, each channel's irradiance can be given by
\begin{equation}\label{eq:dichromatic2}
L^I=m_\mathsf{f}C_\mathsf{f}^I+m_\mathsf{b}C_\mathsf{b}^I,
\end{equation}
where $I=R,G,B$.
The spectral term can be expressed by illumination $E=(E^R,E^G,E^B)$ and interface/body spectral reflectances $S_\mathsf{f}=(S^N,S^N,S^N)$ and $S_\mathsf{b}=(S^R,S^G,S^B)$, and hence
\begin{equation}\label{eq:dichromatic5}
L^I=m_\mathsf{f}S^NE^I+m_\mathsf{b}S^IE^I.
\end{equation}

We define the mixed reflectance as
\begin{equation}\label{eq:mixed1}
P^I\equiv \frac{L^I}{E^I}=m_\mathsf{f}S^N+m_\mathsf{b}S^I.
\end{equation}
Letting $\tilde{S}^N\equiv m_\mathsf{f}S^N$ (interface reflectance) and $\tilde{S}^I\equiv m_\mathsf{b}S^I$ (body reflectance) gives
\begin{equation}\label{eq:mixed2}
P^I=\tilde{S}^N+\tilde{S}^I.
\end{equation}
Now we define body neuter (non-negative) as
\begin{equation}\label{eq:body_neuter1}
\eta\equiv\frac{1}{3}\sum_J\tilde{S}^J,
\end{equation}
and body essence as
\begin{equation}\label{eq:essence1}
\mathcal{S}^I\equiv\tilde{S}^I-\eta,
\end{equation}
where $I,J=R,G,B$.
For highlight removal, we need a known entity invariant to interface reflection.
From Eqs.~\ref{eq:body_neuter1} and \ref{eq:essence1}, it is known that body essence is invariant to interface reflection (has no portion of interface reflection), and of two degrees of freedom (it is of three channels and zero mean) unlike hue and maximum chromaticity (both, one degree of freedom).

From Eqs.~\ref{eq:mixed2} and \ref{eq:essence1},
\begin{equation}\label{eq:body_neuter2}
P^I=\tilde{S}^N+\eta+\mathcal{S}^I.
\end{equation}
Hence we get
\begin{equation}\label{eq:body_neuter3}
P^I-\mathcal{S}^I=\tilde{S}^N+\eta.
\end{equation}
Since $\tilde{S}^N+\eta$ is constant with respect to $I=R,G,B$, $P^I-\mathcal{S}^I$ is spectrally neutral.
Thus we define that entity as (mixed) neuter,
\begin{equation}\label{eq:mixed_neuter1}
\mathcal{P}\equiv P^I-\mathcal{S}^I,
\end{equation}
and we get Lemma~\ref{lem:mixed_neuter}.
\

\begin{lem}[\bf Mixed neuter]\label{lem:mixed_neuter}
Let $P^I$ and $\mathcal{S}^I$ where $I$ is a spectral channel index (e.g. $I=R,G,B$) be the mixed reflectance and body essence, respectively, under dichromatic reflection assumption.
Then $\mathcal{P}\equiv P^I-\mathcal{S}^I$ is constant with respect to $I$, i.e. spectrally neutral.
\end{lem}

\begin{proof}
From the definitions of the mixed reflectance and body essence,
$P^I=\tilde{S}^N+\tilde{S}^I$,
and
$\mathcal{S}^I\equiv\tilde{S}^I-\frac{1}{n_\mathrm{c}}\sum_J\tilde{S}^J$
where $\tilde{S}^N\equiv m_\mathsf{f}S^N$, $\tilde{S}^I\equiv m_\mathsf{b}S^I$, $n_\mathrm{c}$ is the number of color channels, and $J$ is a spectral channel index.
Hence
$P^I-\mathcal{S}^I=\tilde{S}^N+\frac{1}{n_\mathrm{c}}\sum_J\tilde{S}^J$,
thus
$P^I-\mathcal{S}^I$ is constant with respect to $I$, i.e. spectrally neutral.
\end{proof}


From Eq.~\ref{eq:mixed_neuter1}, we meet the core of BREN model, a novel intuitive expression of dichromatic reflection,
\begin{equation}\label{eq:mixed_neuter2}
P^I=\mathcal{P}+\mathcal{S}^I,
\end{equation}
which demonstrates that the mixed reflectance is the sum of the body essence and (mixed) neuter.

Mixed neuter can be easily computed from a mixed reflectance. Lemma~\ref{lem:compute_mixed_neuter} presents how to compute the mixed neuter given a mixed reflectance.

\begin{lem}[\bf Computation of mixed neuter]\label{lem:compute_mixed_neuter}
Let $P^I$ and $\mathcal{S}^I$ where $I$ is a spectral channel index (e.g. $I=R,G,B$) be the mixed reflectance and body essence, respectively, under dichromatic reflection assumption.
Then $\mathcal{P}\equiv P^I-\mathcal{S}^I$ is the mean of $P^I$ with respect to $I$, i.e.
\begin{equation}\label{eq:mixed_neuter3}
\mathcal{P}=\frac{1}{n_\mathrm{c}}\sum_I P^I,
\end{equation}
where $n_\mathrm{c}$ is the number of color channels.
\end{lem}

\begin{proof}
Summation of Eq.~\ref{eq:mixed_neuter1} with respect to $I$ gives $n_\mathrm{c}\mathcal{P}=\sum_I P^I-\sum_I \mathcal{S}^I$. Since body essence is of zero mean as presented earlier (Eqs.~\ref{eq:body_neuter1} and \ref{eq:essence1}), $\mathcal{P}=\frac{1}{n_\mathrm{c}}\sum_I P^I$.
\end{proof}

Actually Lemma 2 implies Lemma 1. Since the mixed neuter is equivalent to the mean of mixed reflectance with respect to $I$, it is constant with respect to $I$.
According to Eqs.~\ref{eq:mixed_neuter2} and \ref{eq:mixed_neuter3}, body essence can be directly calculated from the mixed reflectance:
\begin{equation}\label{eq:essence_comp}
\mathcal{S}^I=P^I-\mathcal{P},
\end{equation}
and so it can be considered a known entity if the mixed reflectance is known (this is one of usual assumptions in separation of reflection components from a single image).

We assume that the illumination is known or properly estimated.
One trivial scheme to estimate illumination from a single image is averaging each channel's intensity for all pixels. With Eq.~\ref{eq:mixed1}, the mixed reflectance is calculated for each pixel from the RGB intensities and illumination.
Then based on the mixed reflectance, the mixed neuter and body essence are calculated for each pixel by Eqs.~\ref{eq:mixed_neuter3} and \ref{eq:essence_comp}.



In most cases, only a small portion of an image has specular highlights.
Therefore we mostly do not need to consider all pixels in an image for highlight removal.
Consequently we consider only high-neuter pixels, practically, pixels whose mixed neuter is larger than a threshold.
If body reflectance ($\tilde{S}^I$) is constant in a finite region, in the region, $\Delta\tilde{S}^N=\Delta P^I$ from Eq.~\ref{eq:mixed2}, $\eta$ and $\mathcal{S}^I$ are constant from Eqs.~\ref{eq:body_neuter1} and \ref{eq:essence1}, hence $\Delta\tilde{S}^N=\Delta P^I=\Delta\mathcal{P}$ from Eq.~\ref{eq:mixed_neuter2}.
That is, reducing the mixed neuter is equivalent to reducing the interface reflection component for a region of constant body reflectance.
For that reason, we use the mixed neuter to reduce the interface reflectance since interface reflectance is unknown.
For each high-neuter pixel, we consider iterative highlight removal, given as
\begin{equation}\label{eq:neuter_demotion}
\mathcal{P}_{k+1}=\mathcal{P}_k+\Delta\mathcal{P}_k,
\end{equation}
where $k$ denotes the iteration number.
The above equation iteratively reduces the mixed neuter (say, neuter demotion) whenever $\Delta\mathcal{P}_k<0$.
$\Delta\mathcal{P}_k$ is determined as follows.
For the high-neuter pixel, we compute the gradients of mixed neuter along the eight-connected pixels:
\begin{equation}\label{eq:neuter_gradient}
\Delta_i\mathcal{P}=\mathcal{P}(\mathbf{p}_i)-\mathcal{P}(\mathbf{p}),
\end{equation}
where $\mathbf{p}$ is the high-neuter pixel location, and $\mathbf{p}_i$ is an eight-connected pixel location of $\mathbf{p}$.
One may choose other types of neighborhood relations (e.g. four-connectivity) instead of eight-connectivity, tradeoffing the computational cost and possibility of finding good neighbors.

We assume that if $\mathcal{S}^I$ is constant in a region, $\tilde{S}^I$ is also constant in the region (hence so is $\eta$).
Therefore we use $\mathcal{S}^I$ to evaluate the closeness of any two body reflectances since $\tilde{S}^I$ is unknown. Table~\ref{tab:reflection} shows constancies and changes of related entities in changes of shading and specularity (interface reflection). From the table, it is known that in specularity change, unlike in shading change, body essence is completely suitable for identifying body reflection colors rather than chromaticity (or a portion of it such as hue and maximum chromaticity) is.
\begin{table}
\caption{Change and constancy of reflection from a surface with constant reflectance}
\centering 
\begin{tabular}{ c c c | c c c | c c c }
  \hline
   &  &  &  & Shading change &  &  & Specularity change &  \\
  \hline

  \hline
   & Constancy &  &  & Chromaticity &  &  & Body essence &  \\

  \hline
   & Change &  &  & Luminance &  &  & Mixed neuter &  \\

  \hline
\end{tabular}\label{tab:reflection}
 \end{table}
Three entities, hue, maximum chromaticity, and body essence are all invariant to interface reflection.
However body essence has two degrees of freedom by its definition while hue and maximum chromaticity have only one degree of freedom.
Therefore, body essence provides more information of body reflection than the other two entities do.
Since body essence is of two degrees of freedom, it is more insensitive to noise than hue and maximum chromaticity are.
Furthermore, body essence is proper for colors around CMY as well as RGB, contrary to maximum chromaticity, which by its definition cannot contain multiple comparable values from color channels simultaneously.

Among the eight-connected pixels, we only consider ones whose mixed neuter is smaller than the mixed neuter of the current pixel, and compute the gradients of body essence along the eight-connected pixels:
\begin{equation}\label{eq:essence_gradient}
\Delta_i\mathcal{S}=\mathcal{S}(\mathbf{p}_i)-\mathcal{S}(\mathbf{p}),
\end{equation}
where $\mathcal{S}=\left(\mathcal{S}^R,\mathcal{S}^G,\mathcal{S}^B\right)$.
Then we set $\Delta\mathcal{P}_k$ for the minimum of gradient of mixed neuter weighted by a Gaussian of the essence similarity:
\begin{equation}\label{eq:neuter_demotion2}
\Delta\mathcal{P}_k=\min_{i\in\left\{i|\Delta_i\mathcal{P}_k<0\right\}} e^{-\lambda\left\|\Delta_i\mathcal{S}_k\right\|^2}\Delta_i\mathcal{P}_k.
\end{equation}
%
%
%
The above Gaussian function replaces a conventional thresholding operation, and provides more appropriate use of body reflectance similarity in highlight removal.

We have tested the proposed method for various images, and provide results of five input images.
Three input images (Head, Fish, and Toys) are downloaded from a webpage of R. T. Tan (\href{http://php-robbytan.rhcloud.com/code.html}{http://php-robbytan.rhcloud.com/code.html}), and the other two images (Dinosaur and Mickey-ball) are newly captured.

\begin{figure}
\centering
\subfigure[]{\epsfig{file=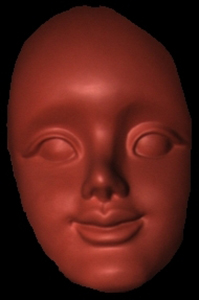,height=2.48cm}}
\subfigure[]{\epsfig{file=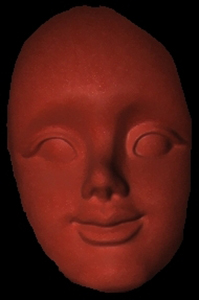,height=2.48cm}}
\subfigure[]{\epsfig{file=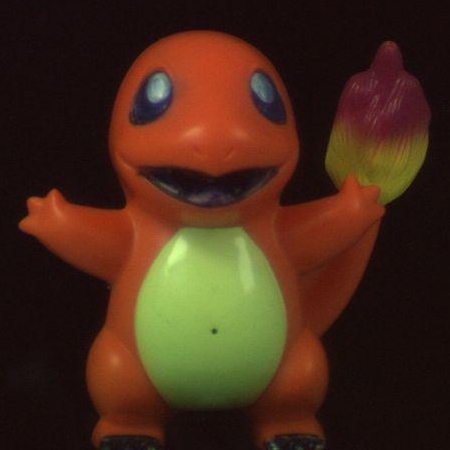,height=2.48cm}}
\subfigure[]{\epsfig{file=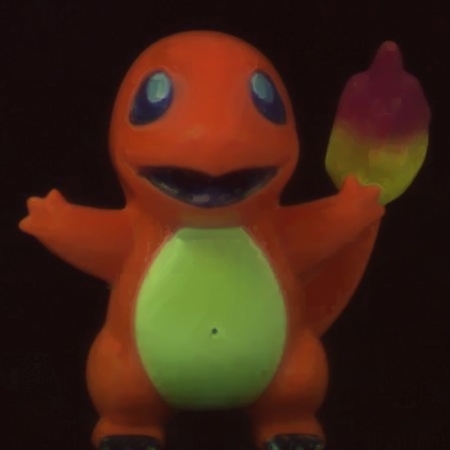,height=2.48cm}}
\caption{{\bf Results of Head and Dinosaur.} (a) Input Head image (R. T. Tan's) and (b) the result. (c) Input Dinosaur image and (d) the result.}\label{fig:result_head_dinosaur}
\end{figure}


Figure~\ref{fig:result_head_dinosaur}a--b show the result of a single-colored object, the Head. We can see that the specular component is adequately removed.
Figure~\ref{fig:result_head_dinosaur}c--d show the result of the Dinosaur. Even though there is smooth color gradation in its tail, the specular component is removed quite well without undesirable artifact.

\begin{figure}
\centering
\subfigure[]{\epsfig{file=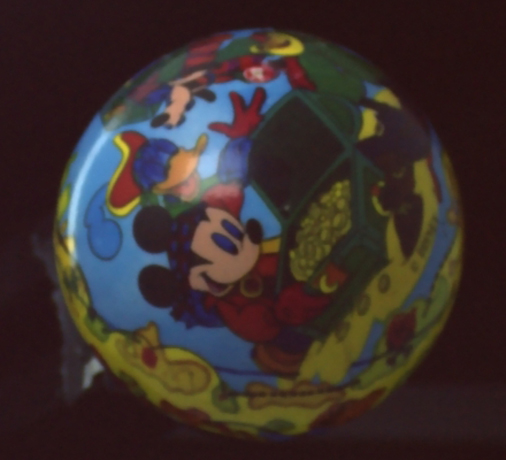,height=3.8cm}}
\subfigure[]{\epsfig{file=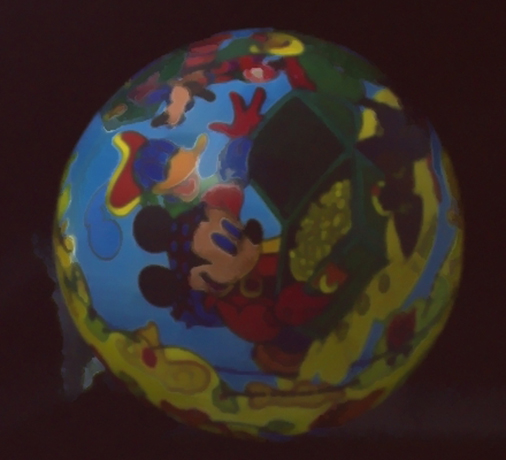,height=3.8cm}}
\caption{{\bf Result of Mickey-ball.} (a) Input image and (b) result.}\label{fig:result_globe}
\end{figure}

\begin{figure}
\centering
\subfigure[]{\epsfig{file=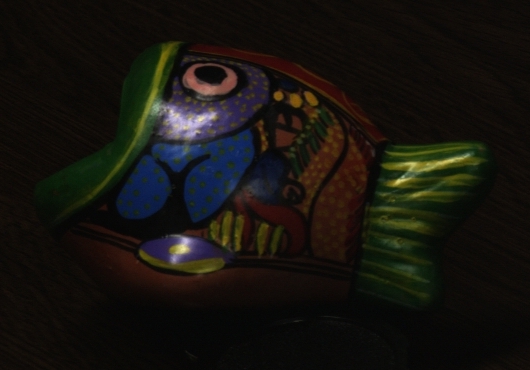,height=2.9cm}}
\subfigure[]{\epsfig{file=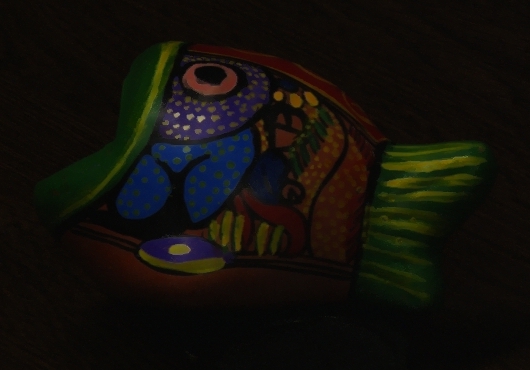,height=2.9cm}}
\caption{{\bf Result of Fish.} (a) Input image (R. T. Tan's) and (b) result.}\label{fig:result_fish}
\end{figure}

\begin{figure}
\centering
\subfigure[]{\epsfig{file=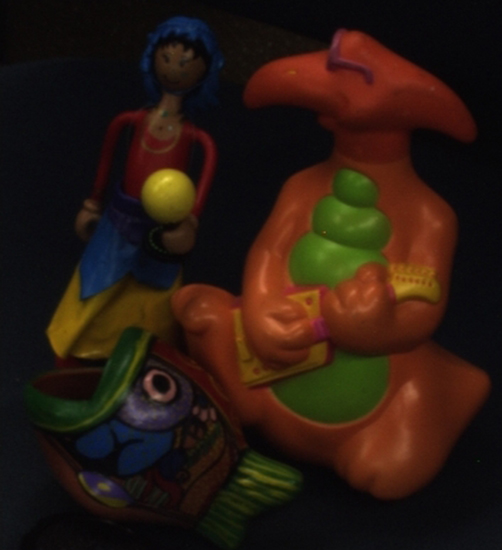,height=4.6cm}}
\subfigure[]{\epsfig{file=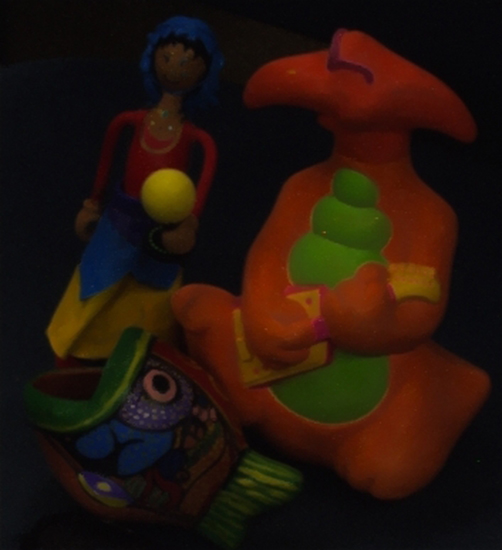,height=4.6cm}}
\caption{{\bf Result of Toys.} (a) Input image (R. T. Tan's) and (b) result.}\label{fig:result_toys}
\end{figure}

Figures~\ref{fig:result_globe}, \ref{fig:result_fish} and \ref{fig:result_toys} show the results of more complicated scenes, the Mickey-ball, Fish, and Toys, respectively. Despite their high complexity of colors and textures, the body reflection components are properly estimated.


We presented a novel reflection color model (BREN) and an effective method for separating dichromatic reflection components using a single image.
We showed body essence is a better entity for specular highlight removal than hue and maximum chromaticity are.
The Gaussian coefficient generalizes a conventional simple thresholding scheme, and it provides detailed use of body color similarity.
The proposed method does not require any region segmentation, and thus it does not depend on segmentation accuracy.


%
%
%

\

The authors are grateful to Professor Robby T. Tan for his images we have used here.
This research was supported by Basic Science Research Program through the National Research Foundation of Korea (NRF) funded by the Ministry of Education (No. 2012R1A1A2009138).

\bibliographystyle{osajnl}
\bibliography{reflection}

%
%
%
%

\end{document}